\documentclass[lettersize,journal]{IEEEtran}

\usepackage{graphicx} 
\usepackage[algo2e,ruled]{algorithm2e}
\usepackage{amsfonts}
\usepackage{mathtools}
\usepackage{amssymb}
\usepackage{amsthm}
\usepackage{xcolor}
\usepackage{hyperref}
\usepackage{dsfont}
\usepackage{natbib}
\usepackage{cancel}
\usepackage{algorithm}
\usepackage{bbm}
\usepackage{graphicx}
\usepackage{subcaption}
 \usepackage{algpseudocode}
 \usepackage{algorithmicx}
 \algdef{SE}[DOWHILE]{Do}{doWhile}{\algorithmicdo}[1]{\algorithmicwhile\ #1}%

\algrenewcommand\algorithmicrequire{\textbf{Input:}}
\algrenewcommand\algorithmicensure{\textbf{Output:}}

\newcommand{\R}{\mathbb R}

\def\bA{\boldsymbol{A}}
\def\bB{\boldsymbol{B}}
\def\bC{\boldsymbol{C}}

\def\bI{\boldsymbol{I}}

\def\bL{\boldsymbol{L}}

\def\cA{\mathcal{A}}

\def\ba{\boldsymbol{a}}

\def\bu{\boldsymbol{u}}

\def\bx{\boldsymbol{x}}

\def\bzero{\boldsymbol{0}}
\def\bone{\boldsymbol{1}}

\def\bSigma{\boldsymbol{\Sigma}}

\def\sign{\mathrm{sign}}

\def\bbone{\mathds{1}}

\title{A Smoothing Newton Method for Rank-one Matrix Recovery}
\author{Tyler Maunu \thanks{T. Maunu is with the Department of Mathematics at Brandeis University.}
 Gabriel Abreu \thanks{G. Abreu is with the Department of Mathematics at Brandeis University.}}

\newtheorem{theorem}{Theorem}
\newtheorem{proposition}{Proposition}
\newtheorem{lemma}{Lemma}

\newtheorem{assumption}{Assumption}

\begin{document}

\maketitle

\begin{abstract}
We consider the phase retrieval problem, which involves recovering a rank-one positive semidefinite matrix from rank-one measurements. A recently proposed algorithm based on Bures-Wasserstein gradient descent (BWGD) exhibits superlinear convergence, but it is unstable, and existing theory can only prove local linear convergence for higher rank matrix recovery. We resolve this gap by revealing that BWGD implements Newton's method with a nonsmooth and nonconvex objective. We develop a smoothing framework that regularizes the objective, enabling a stable method with rigorous superlinear convergence guarantees. Experiments on synthetic data demonstrate this superior stability while maintaining fast convergence.
\end{abstract}


\section{Introduction}

Phase retrieval—the problem of recovering a real or complex signal from magnitude-only measurements—is a fundamental problem in signal processing. Its applications range from X-ray crystallography to astronomical imaging, where measurement systems capture a form of intensity \citep{harrison1993phase,fienup1982phase,fienup1987phase,miao1998phase}. The seemingly simple constraint of measuring magnitudes transforms what would be a linear problem into a challenging nonlinear and nonconvex optimization problem.

Mathematically, phase retrieval seeks to reconstruct a signal $\bx \in \mathbb{R}^d$ (or $\mathbb{C}^d$) from measurements of the form $y_i^2 = |\langle \ba_i, \bx \rangle|^2$ for $i = 1, \ldots, n$, where $\ba_i$ are known sensing vectors. The loss of phase information creates a fundamental ambiguity: if $\bx$ is a solution, so is $-\bx$ (and in the complex case, $e^{i\theta}\bx$ for any $\theta \in \R$). More critically, the direct optimization formulation using least squares yields a nonconvex objective function, making it difficult to solve effectively.

The phase retrieval problem is a specific instance of a broader class of low-rank matrix sensing problems that arise throughout signal processing \citep{recht2010guaranteed}. By recognizing that $|\langle \ba_i, \bx \rangle|^2 = \langle \ba_i \ba_i^T, \bx\bx^T \rangle$ (note here we focus on the real-valued case), phase retrieval can be cast as recovering a rank-one positive semidefinite matrix from rank-one linear measurements. This perspective connects phase retrieval to other problems such as matrix completion and blind deconvolution \citep{ma2020implicit}. The study of low-rank matrix sensing problems has resulted in a greater understanding of the benign behavior of nonconvex optimization algorithms. This study has then had further implications for machine learning more broadly - see, for example, the line of works on deep matrix factorizations \citep{arora2019implicit}.

Two main algorithmic approaches have emerged for phase retrieval. Convex relaxation methods, pioneered by \cite{candes2013phaselift}, create convex analogs of the nonconvex problem by lifting the problem to a higher-dimensional space. This often results in a substantial increase in computational cost. Alternatively, the other main approach is to formulate nonconvex optimization methods directly. These nonconvex methods work directly with the natural problem formulation and try to exploit benign energy landscapes that emerge under certain statistical assumptions \citep{candes2015phase,chi2019nonconvex,ma2020implicit,maunu2024acceleration}. While nonconvex approaches are computationally more efficient, their theoretical guarantees typically require strong assumptions, such as Gaussian sensing vectors.

Recently, \cite{maunu2023bures} introduced a novel perspective on the phase retrieval problem by formulating it as a Bures-Wasserstein barycenter problem. This geometric viewpoint leads to an efficient gradient descent algorithm, called Bures-Wasserstein Gradient Descent (BWGD), that has compelling theoretical properties. However, a significant limitation emerges: the theoretical guarantees only hold for matrices of rank $r \geq 5$, primarily due to moment conditions on inverse $\chi^2$ random variables. This excludes the phase retrieval case where $r = 1$, which is of great practical interest. It was also noted in \cite{maunu2023bures} that BWGD exhibits superlinear convergence for the $r=1$ case, although convergence is unstable. Therefore, the lack of theoretical understanding of this regime represents a critical barrier to further algorithmic development from this perspective.

In this paper, we directly address the rank-one case and explain the superlinear convergence phenomenon. Our key insight is that BWGD for rank-one matrices corresponds exactly to Newton's method applied to an amplitude-based optimization problem. This objective is nonsmooth, which explains the instability in BWGD and prevents standard analysis techniques from applying. We resolve this challenge through a dynamic smoothing framework that regularizes the objective while maintaining the essential structure that enables fast convergence.

\subsection{Contributions}

The theoretical and algorithmic contributions of our work are as follows.

\begin{itemize}
    \item \textbf{Connecting BWGD and Newton's Method}: We establish that BWGD for real-valued rank-one matrices is equivalent to Newton's method on an amplitude objective for the phase retrieval problem. This explains the emergence of superlinear convergence observed in \cite{maunu2023bures}.
    
    \item \textbf{Novel smoothing framework}: We develop a principled approach to handle the nonsmoothness of the amplitude-based objective that differs from past work \cite{gao2020perturbed} and is inspired by smoothed Wasserstein distances \cite{goldfeld2020gaussian}. Our dynamic smoothing approach allows us to stabilize the convergence of the Newton method and bypass the $r \geq 5$ restriction of \cite{maunu2023bures} in our analysis.
    
    \item \textbf{Convergence analysis}: We prove superlinear convergence for our smoothed Newton method under standard assumptions of Gaussian measurements and spectral initialization. Our analysis shows how carefully tuning the dynamic smoothing parameter allows one to control the iterates of BWGD despite the global nonconvexity of the problem.
    
    \item \textbf{Efficient implementation}: Despite the Newton-like updates, our method maintains the $O(nd)$ per-iteration complexity of gradient descent, requiring only an initial $O(d^3)$ Cholesky decomposition that is used for whitening. This compares favorably to similar Gauss-Newton methods \cite{gao2017phaseless} that require $O(nd^2 + d^3)$ operations \emph{per iteration}.
    
    \item \textbf{Empirical validation}: We demonstrate through experiments that our method achieves superior stability compared to unregularized BWGD while maintaining computational efficiency. We demonstrate how different variants of the dynamic smoothing are able to navigate the trade-off between stability and convergence speed.
\end{itemize}

Our work makes fundamental contributions at the intersection of benign nonconvex optimization and nonsmooth higher-order optimization methods. While the phase retrieval landscape exhibits favorable geometry that enables first-order methods to succeed, extracting superlinear convergence from this structure has remained elusive. We bridge this gap by introducing a new class of dynamic smoothing techniques specifically designed for Newton-type methods in nonconvex settings. Unlike traditional smoothing approaches that try to solve a KKT system using smoothing \citep{chen1998global}, our framework smooths the original objective. This allows for a principled transition from stable linear convergence to rapid superlinear convergence. This represents both a theoretical advance—providing rare convergence guarantees for higher-order methods in nonconvex optimization—and a practical advance, as our smoothing strategy maintains computational efficiency while achieving superior empirical performance.

\subsection{Notation}

We use bold lowercase letters for vectors and bold uppercase letters for matrices. The set of $d \times d$ positive semidefinite matrices is denoted $\mathbb{S}^{d}_+$, and positive definite matrices are $\mathbb{S}_{++}^d$. For a vector $\bx \in \mathbb{R}^d$, $\|\bx\|$ denotes the Euclidean norm. For matrices, $\|\cdot\|_2$ denotes the spectral norm and $\|\cdot\|_F$ denotes the Frobenius norm. The notation $\bA \succeq \bB$ means $\bA - \bB$ is positive semidefinite. The indicator function for a set denoted by $\mathbbm{1}(\cdot)$.

\section{Background}

In this section, we provide the mathematical foundation for our approach. We begin with the problem formulation in Section \ref{sec:problem_formulation}, establishing the connection between phase retrieval and Bures-Wasserstein barycenters. In Section \ref{sec:newton_connection}, we reveal our first key insight: BWGD corresponds to Newton's method on a nonsmooth objective. Finally, Section \ref{sec:related_work} surveys related approaches and positions our contribution within the broader literature.

\subsection{Problem Formulation}
\label{sec:problem_formulation}

In this paper, we focus on the real phase retrieval problem. The real phase retrieval problem seeks to recover a signal $\bu_\star \in \mathbb{R}^d$ from magnitude-only measurements:
\begin{equation}
    y_i = |\ba_i^T \bu_\star|, \quad i = 1, \ldots, n
\end{equation}
where $\ba_1, \ldots, \ba_n \in \mathbb{R}^d$ are known sensing vectors. By recognizing that $|\ba_i^T \bu_\star|^2 = \langle \ba_i \ba_i^T, \bu_\star \bu_\star^T \rangle$, we can view this as recovering a rank-one positive semidefinite matrix from linear measurements.

A natural approach is to minimize the amplitude-based loss:
\begin{equation}
\label{eq:amplitude_min}
    \min_{\bu \in \mathbb{R}^d} \frac{1}{2n} \sum_{i=1}^n (|\ba_i^T \bu| - y_i)^2.
\end{equation}
However, the objective in this problem is nonconvex and nonsmooth. Indeed, derivatives fail to exist when $\ba_i^T \bu = 0$, which creates challenges for optimization algorithms.

A key insight from \cite{maunu2023bures} is that phase retrieval can be reformulated using optimal transport theory. The Bures-Wasserstein distance between positive semidefinite matrices $\bSigma_0, \bSigma_1 \in \mathbb{S}_+^d$ is:
\begin{equation}
    d_{BW}^2(\bSigma_0, \bSigma_1) = \text{Tr}(\bSigma_0) + \text{Tr}(\bSigma_1) - 2\text{Tr}((\bSigma_0^{1/2}\bSigma_1\bSigma_0^{1/2})^{1/2}).
\end{equation}

This distance has a natural interpretation: it equals the Wasserstein distance between zero-mean Gaussians with covariances $\bSigma_0$ and $\bSigma_1$. For rank-one matrices, it provides a geometrically meaningful way to measure distances in the space of possible solutions.

The phase retrieval problem can be cast as finding the Bures-Wasserstein barycenter:
\begin{equation}
\label{eq:bw_barycenter}
    \min_{\substack{\bSigma \in \mathbb{S}_+^d \\ \mathrm{rank}(\bSigma) = 1}} \frac{1}{2n} \sum_{i=1}^n d_{BW}^2(\bSigma, y_i \ba_i \ba_i^T).
\end{equation}
Indeed, when the sensing vectors satisfy $\frac{1}{n} \sum_{i=1}^n \ba_i \ba_i^T = \bI$, this barycenter formulation becomes equivalent to:
\begin{equation}
\label{eq:sqrt_loss}
    \min_{\substack{\bSigma \in \mathbb{S}_+^d \\ \mathrm{rank}(\bSigma) = 1}} \frac{1}{2n}\sum_{i=1}^n \left(\sqrt{\langle \ba_i \ba_i^T, \bSigma \rangle} - \sqrt{y_i}\right)^2.
\end{equation}
In the rank one case, this is equivalent to \eqref{eq:amplitude_min}.

The BWGD algorithm for solving this problem takes the form:
\begin{equation}
\label{eq:bwgd_update}
    \bu_{t+1} = (1-\eta) \bu_t + \frac{\eta}{n} \sum_{i=1}^n \frac{y_i \ba_i \ba_i^T \bu_t}{|\ba_i^T \bu_t|}
\end{equation}
for a step size $\eta \in [0,1]$.

When the condition $\frac{1}{n} \sum_{i} \ba_i \ba_i^T = \bI$ does not hold, we can enforce it through a whitening transformation. Let $\bC_{\mathcal{A}} = \frac{1}{n} \sum_{i} \ba_i \ba_i^T$ be the empirical covariance. We transform the sensing vectors as $\bL \tilde{\ba}_i = \ba_i$, where $\bC_{\cA} = \bL \bL^T$ is the Cholesky decomposition, and solve the problem with the transformed vectors $\tilde \ba_i$. After recovering the solution to the transformed problem $\tilde \bu_\star$, we can recover the original vector $\bL \bu_\star = \tilde \bu_\star$. This whitening approach maintains the essential structure while ensuring the algorithm's applicability to general measurement settings.

\subsection{Connection to Nonsmooth Newton's Method}
\label{sec:newton_connection}

We now reveal a fundamental connection that explains the empirical success of BWGD. When the sensing vectors satisfy $\frac{1}{n} \sum_i \ba_i \ba_i^T = \bI$ and we set the step size $\eta = 1$, BWGD becomes exactly Newton's method applied to the amplitude-based objective \eqref{eq:amplitude_min}.

To see this, consider the objective:
\begin{equation}\label{eq:amplitude_loss}
  F(\bu) = \frac{1}{2n}\sum_{i=1}^n (|\ba_i^T \bu|-y_i)^2
\end{equation}

At points where $\ba_i^T \bu \neq 0$ for all $i$, we can compute:
\begin{align}
    \nabla F(\bu) &= \frac{1}{n} \sum_{i=1}^n ( |\ba_i^T \bu|-y_i) \frac{\ba_i \ba_i^T \bu}{|\ba_i^T \bu|}\\
    \nabla^2 F(\bu) &= \frac{1}{n} \sum_i \ba_i \ba_i^T \\ \nonumber
    &+ ( |\ba_i^T \bu|-y_i) \left[\frac{\ba_i \ba_i^T}{|\ba_i^T \bu|} - \frac{\ba_i \ba_i^T (\ba_i^T \bu)^2}{|\ba_i^T \bu|^3} \right]
\end{align}
When $\frac{1}{n} \sum_i \ba_i \ba_i^T = \bI$, the $\bu$-dependent terms vanish, yielding $\nabla^2 F(\bu) = \bI$. The Newton update becomes:
\begin{align}
    \bu_{t+1} &= \bu_t - \nabla F(\bu_t) \\
    &= \bu_t - \frac{1}{n} \sum_{i=1}^n ( |\ba_i^T \bu_t|-y_i) \frac{\ba_i \ba_i^T \bu_t}{|\ba_i^T \bu_t|} \\
    &= \frac{1}{n} \sum_{i=1}^n \frac{y_i \ba_i \ba_i^T \bu_t}{|\ba_i^T \bu_t|}
\end{align}

This is precisely the BWGD update \eqref{eq:bwgd_update} with $\eta = 1$. This connection explains the superlinear convergence observed empirically—Newton's method naturally exhibits quadratic convergence near the solution.

However, a critical challenge emerges: the objective $F$ is nonsmooth. Derivatives fail to exist when $\ba_i^T \bu = 0$ for some $i$, and the gradient is not locally Lipschitz continuous. This prevents the application of standard convergence analysis for Newton's method or even semismooth Newton methods. Our key contribution is developing a dynamic smoothing framework that regularizes the problem while still allowing fast convergence of Newton's method, enabling rigorous convergence guarantees.

\subsection{Related Work}
\label{sec:related_work}

Phase retrieval has been extensively studied due to its importance in applications ranging from X-ray crystallography \citep{harrison1993phase} to astronomical imaging \citep{fienup1987phase}. We organize the related work by algorithmic approaches and stick to the most relevant advances for our work.

\textbf{Convex relaxation methods.} \cite{candes2013phaselift} pioneered the PhaseLift approach, lifting the problem to the space of positive semidefinite matrices and solving a convex relaxation. While providing global guarantees, these methods suffer from high computational cost. Recent work has developed efficient sketching algorithms to reduce this burden \citep{yurtsever2021scalable}.

\textbf{Nonconvex gradient methods.} The discovery that phase retrieval exhibits benign nonconvex geometry under Gaussian measurements led to efficient gradient-based algorithms. \cite{candes2015phase} introduced Wirtinger Flow, achieving local convergence from spectral initialization. Subsequent work improved sample complexity \citep{chen2015solving} and extended to more general settings \citep{ma2020implicit}. Other recent work \cite{maunu2024acceleration} showed that gradient descent for phase retrieval can be accelerated.

\textbf{Newton-type methods.} \cite{gao2017phaseless} developed a Gauss-Newton method for real phase retrieval using the amplitude-based cost, achieving quadratic convergence but requiring $O(nd^2 + d^3)$ operations per iteration. \cite{huang2024asymptotic} extended this to complex signals. Our method achieves similar convergence with only $O(nd)$ per-iteration cost after an initial $O(d^3)$ preprocessing step, but only deals with real signals and real sensing vectors. 

\textbf{Smoothing and regularization approaches.} The idea of smoothing nonsmooth objectives has been explored in various contexts. \cite{chen1998global} developed general smoothing Newton methods, while \cite{chen2020proximal} applied semismooth Newton methods to manifold optimization. In the context of iteratively reweighted least squares (IRLS), dynamic smoothing—where regularization decreases over iterations—has proven effective \citep{daubechies2010iteratively, mukhoty2019globally}. Recently, \cite{lerman2025global} applied similar ideas to robust subspace estimation.

\textbf{Bures-Wasserstein methods.} \cite{maunu2023bures} introduced the optimal transport perspective for low-rank matrix recovery, developing the BWGD algorithm. While their theoretical analysis requires rank $r \geq 5$, our work extends these guarantees to all ranks by revealing and addressing the underlying Newton structure. Optimization over Bures-Wasserstein space has been considered more broadly in the contexts of variational inference \citep{lambert2022variational,diao2023forward} and generative models \citep{brechet2023critical}.

Our contribution synthesizes these threads: we explain BWGD's success through its Newton interpretation, develop a principled smoothing approach to handle nonsmoothness, and provide unified convergence guarantees for all ranks while maintaining computational efficiency.

\section{Method}

While BWGD has a Newton interpretation as outlined in Section \ref{sec:newton_connection}, the nonsmooth objective prevents direct application of standard Newton analysis. In this section we develop a principled dynamic smoothing approach that allows us to analyze BWGD for phase retrieval. Our solution regularizes the objective in a way that enables rigorous convergence guarantees.

We begin in Section \ref{subsec:objective} by introducing our dynamically regularized objective. Section \ref{subsec:whiten} shows how whitening connects this regularized objective to the Bures-Wasserstein formulation. Finally, Section \ref{subsec:algo} presents our complete algorithm with dynamic adjustment of the smoothing parameter.

\subsection{Regularized Objective Function}
\label{subsec:objective}

To address the nonsmoothness of the objective \eqref{eq:amplitude_loss} when $\ba_i^T \bu = 0$, we introduce a regularized amplitude objective. Our approach builds on the perturbed formulation of \cite{gao2020perturbed} but with a crucial modification that enables our convergence analysis.

We recall that the standard perturbed objective takes the form:
\begin{equation}\label{eq:smooth_obj_standard}
    \tilde{F}_\epsilon(\bu) = \frac{1}{2n} \sum_{i=1}^n \left(\sqrt{(\ba_i^T \bu)^2 + \epsilon y_i^2} - y_i\sqrt{1 + \epsilon}\right)^2,
\end{equation}
where \cite{gao2020perturbed} introduce smoothing proportional to the measurements $y_i$.

Our key innovation is to instead scale the regularization by the sensing vector norms:
\begin{equation}\label{eq:smooth_obj}
    F_\epsilon(\bu) = \frac{1}{2n} \sum_{i=1}^n \left(\sqrt{(\ba_i^T \bu)^2 + \epsilon \|\ba_i\|^2} - \sqrt{y_i^2 + \epsilon \|\ba_i\|^2}\right)^2.
\end{equation}

This modification corresponds exactly to the regularization in \cite{maunu2023bures}, which connects our smoothed objective to the Bures-Wasserstein geometry. From the Bures-Wasserstein perspective, this objective corresponds to finding:
\begin{equation*}
    \min_{\bu} \frac{1}{2n} \sum_{i=1}^n d_{BW}^2\left(\bu \bu^T + \epsilon \bI, (y_i^2 + \epsilon \|\ba_i\|^2) \frac{\ba_i \ba_i^T}{\|\ba_i\|^2}\right)
\end{equation*}
Thus, we seek a rank-one approximation to the regularized barycenter, with the regularization vanishing as $\epsilon \to 0$. We thus also view this as regularizing the problem by considering a smoothed Wasserstein barycenter \citep{goldfeld2020gaussian}.

\subsection{Whitened Objective Function}
\label{subsec:whiten}

The BWGD algorithm requires the condition $\frac{1}{n}\sum_i \ba_i \ba_i^T = \bI$ for its Newton interpretation. When this doesn't hold naturally, we enforce it through whitening. While the details of this are present in \cite{maunu2023bures}, for clarity, we discuss how this procedure works in this section.

Let $\bC_{\mathcal{A}} = \frac{1}{n} \sum_{i=1}^n \ba_i \ba_i^T$ denote the empirical covariance. We transform the sensing vectors as $\tilde{\ba}_i = \bC_{\mathcal{A}}^{-1/2} \ba_i$ and work with the whitened objective:
\begin{equation}
    F_{\epsilon,w}(\bu) = \frac{1}{2n} \sum_{i=1}^n \left(\sqrt{(\tilde{\ba}_i^T \bu)^2 + \epsilon \|\tilde{\ba}_i\|^2} - \sqrt{y_i^2 + \epsilon \|\tilde{\ba}_i\|^2}\right)^2
\end{equation}

The following result, adapted from \cite{maunu2023bures}, shows that whitening preserves the problem structure. 
This allows us to optimize in the whitened space and recover the original solution through the inverse transformation. 
\begin{lemma}
    If $y_i = |\ba_i^T \bu_\star|$ for some signal $\bu_\star$, then $\bC_{\mathcal{A}}^{1/2} \bu_\star$ is a stationary point of $F_{\epsilon,w}$. 
\end{lemma}

In practice, it is more numerically stable to whiten the vectors using a Cholesky procedure. Algorithm \ref{alg:whiten} provides this implementation, which avoids explicit computation of inverse matrix square roots.

\begin{algorithm}
\caption{Efficient Whitening via Cholesky Factorization}  
\label{alg:whiten}
\begin{algorithmic}[1]
\Require Sensing vectors $\{\ba_i\}_{i=1}^n \subset \mathbb{R}^d$ with $n > d$
\Ensure Whitened vectors $\{\tilde{\ba}_i\}_{i=1}^n$ satisfying $\frac{1}{n}\sum_i \tilde{\ba}_i \tilde{\ba}_i^T = \bI$,  Cholesky factor $\bL$
\State Compute empirical covariance: $\bC_{\mathcal{A}} = \frac{1}{n} \sum_{i=1}^n \ba_i \ba_i^T$
\State Compute Cholesky factorization: $\bC_{\mathcal{A}} = \bL \bL^T$
\State For each $i$, solve $\bL \tilde{\ba}_i = \ba_i$ via forward substitution
\State \Return Whitened vectors $\{\tilde{\ba}_i\}_{i=1}^n$ and Cholesky factor $\bL$
\end{algorithmic}
\end{algorithm}

\subsection{Algorithm}
\label{subsec:algo}

Our complete algorithm combines gradient descent on the whitened objective with dynamic adjustment of the smoothing parameter. The key insight is that as the iterate approaches the solution, we can reduce $\epsilon$ while maintaining stability. Since we later show that the objective is $1+\epsilon_t$ smooth in Proposition \ref{prop:scsm}, we can set the step size for BWGD to be $1/(1+\epsilon_t)$. This corresponds to a damped Newton step.

While our later theoretical analysis requires $\epsilon_t = 2\sqrt{\|\bu_t - \bu_\star\| / \|\bu_\star\|}$, this quantity is not directly observable in practice. We can therefore use a heuristic based on the empirical loss. To estimate this, we notice that
\begin{align*}
    \frac{1}{n} \sum_i &(|\ba_i^T \bu_t| - |\ba_i^T \bu_\star|)^2 \\
    &= \frac{1}{n} \sum_{\sign =} (\ba_i^T \bu_t - \ba_i^T \bu_\star)^2 + \frac{1}{n} \sum_{\sign \neq } (\ba_i^T \bu_t + \ba_i^T \bu_\star)^2 \\
    &= \Theta( \| \bu_t - \bu_\star\|^2 )
\end{align*}
for $\bu_t$ close to $\bu_\star$. Thus, as a heuristic when $\|\bu_\star\| = 1$, we can take 
\begin{equation}\label{eq:eps_update_theoretical}
    \epsilon_k \propto F(\bu_t)^{1/4}. 
\end{equation}
We call this the \emph{loss heuristic}.

As is done in dynamic smoothing \cite{peng2022global,lerman2025global}, we also empirically test a smoothing rule that depends on the quantiles of the loss:
\begin{equation}\label{eq:eps_update_empirical}
    \epsilon_t = \min\left(\epsilon_{t-1}, \; q_\gamma\left(\{(|\tilde{\ba}_i^T \bu_t| - y_i)^2\}_{i=1}^n\right)\right)
\end{equation}
where $q_\gamma$ denotes the $\gamma$-quantile of the squared residuals:
\begin{equation}
    q_\gamma(\mathcal{B}) = \max\left\{ c \in \mathbb{R} : \frac{|\{b \in \mathcal{B} : b \leq c\}|}{|\mathcal{B}|} \leq \gamma\right\}
\end{equation}
We refer to this method of updating the smoothing parameter as the \emph{quantile heuristic}. In practice, we observe that this update remains relatively stable and achieves faster convergence than the loss heuristic; however, we have no proof of this fact.

In practice, both of these adaptive schemes ensure that the smoothing decreases monotonically: $\epsilon_t \leq \epsilon_{t-1}$. Furthermore, the parameter adapts to the current optimization landscape, in the sense that as $\bu_t \to \bu_\star$ we have that $\epsilon_t \to 0$. Algorithm \ref{alg:bwgd_ds} presents our complete method.

\begin{algorithm}
\caption{BWGD with Dynamic Smoothing (BWGD-DS)}  
\label{alg:bwgd_ds}
\begin{algorithmic}[1]
\Require Sensing vectors $\{\ba_i\}_{i=1}^n$, measurements $\{y_i\}_{i=1}^n$, initial point $\bu_0$, initial smoothing $\epsilon_0$, quantile parameter $\gamma \in (0,1)$
\Ensure Estimate $\hat{\bu}$ of the signal $\bu_\star$
\State Compute whitened vectors $\{\tilde{\ba}_i\}_{i=1}^n$ and factor $\bL$ using Algorithm \ref{alg:whiten}
\State Set $t = 0$
\Do
    \State Update smoothing parameter using \eqref{eq:eps_update_theoretical} or \eqref{eq:eps_update_empirical}
    \State Update iterate:
    \begin{align*}
        \bu_{t+1} = &\Big(1-\frac{1}{1+\epsilon_t}\Big) \bu_t  \\ &+\frac{1}{1+\epsilon_t}\frac{1}{n}  \sum_{i=1}^n \frac{\sqrt{y_i^2 + \epsilon_t \|\tilde \ba_i\|^2}}{\sqrt{|\tilde \ba_i^T \bu_t|^2 + \epsilon_t \|\tilde \ba_i\|^2}}\tilde \ba_i \tilde \ba_i^T \bu_t
    \end{align*}
    \State $t = t + 1$
\doWhile{not converged}
\State Solve $\bL \hat{\bu} = \bu_t$ via forward substitution
\State \Return $\hat{\bu}$
\end{algorithmic}
\end{algorithm}

\textbf{Computational complexity:} The algorithm achieves $O(nd)$ per-iteration complexity—matching gradient descent—after an initial $O(nd^2 + d^3)$ preprocessing cost for whitening. This compares favorably to Gauss-Newton methods that require $O(nd^2 + d^3)$ operations per iteration. The forward substitutions in Algorithms \ref{alg:whiten} and \ref{alg:bwgd_ds} require only $O(d^2)$ operations due to the triangular structure of $\bL$.

\textbf{Implementation details:} The update in step 5 corresponds to BWGD with step size $\eta_t = 1/(1 + \epsilon_t)$. As $\epsilon_t \to 0$, we recover the Newton step $\eta_t \to 1$, explaining the transition from linear to superlinear convergence. The convergence criterion can be based on the relative change in iterates or the objective value. 

\section{Theoretical Analysis}

In this section, we present the theoretical analysis of BWGD-DS. Section \ref{subsec:assump} states the assumptions required for our convergence guarantees. Section \ref{subsec:thm} presents our main convergence theorem that establishes superlinear convergence of BWGD-DS. The proof is given in Section \ref{subsec:proof}, with supporting technical results deferred to Section \ref{subsec:supp_proofs}.

In order not to bury the lede, the main point of our theory is that as we approach the global minimum and we take the regularization to $0$, the objective approaches a quadratic with identity Hessian. For such an objective, the gradient step is equivalent to the Newton step (due to the fact that the Hessian is the identity), and from this we obtain superlinear convergence. 

\subsection{Assumptions}
\label{subsec:assump}

We require three main assumptions for our theoretical analysis. First, we consider a noiseless setting of real phase retrieval from real measurements. A crucial limitation of our current analysis is the restriction to real-valued signals and sensing vectors. Extending our smoothed Newton framework to complex phase retrieval remains an important open problem for future work.
\begin{assumption}[Noiseless Observations]\label{assump:noiseless}
    For a signal $\bu_\star \in \R^d$ and a set of sensing vectors $\cA = \{\ba_i \in \R^d, \ i=1, \dots, n\}$, the observations are $y_i = |\ba_i^T \bu_\star|$.
\end{assumption}

Our second assumption states that the sensing vectors are standard Gaussian. This statistical assumption enables concentration arguments and has been used in all works that study nonconvex optimization for the phase retrieval problem, see, e.g.,  \citep{candes2015phase,gao2017phaseless,ma2020implicit,maunu2024acceleration}. This assumption is also present in convex approaches since it enables various restricted isometry properties \citep{chen2015exact,cai2015rop}.

\begin{assumption}[Gaussian Measurements]\label{assump:gaussian}
    The sensing vectors are i.i.d. standard Gaussian:
    $$\ba_i \overset{i.i.d.}{\sim} N(\bzero, \bI), \ i=1, \dots, n.$$
    The whitened vectors are computed as:
    $$\tilde \ba_i = \bSigma_{\cA}^{-1/2} \ba_i, \ i=1, \dots, n.$$
\end{assumption}

 For theoretical convenience, we assume whitening via the inverse square root rather than the Cholesky factorization used in our implementation (Algorithm \ref{alg:bwgd_ds}).

Third, we require initialization within a neighborhood of the true signal:
\begin{assumption}[Initialization]\label{assump:init}
    The algorithm is initialized such that 
    $$\|\bu_0 - \bu_\star\| \leq \delta \|\bu_\star\|,$$
    where $\delta \leq O(1/\log^8(d))$.
\end{assumption}
This initialization accuracy can be achieved through spectral methods with a large enough sample size, as is standard in the nonconvex phase retrieval literature \cite{gao2017phaseless,ma2020implicit}. We note that our constants and the dependence on the log-dimension factor are most likely not optimal. Since the dimension dependence in the initialization assumption increases the required sample size for a good initialization, it may be more effective in practice to run a two-stage algorithm that employs gradient descent in the first stage and then utilizes Newton's method in the second stage. We leave the exploration of this idea to future work.

\subsection{Main Theorem}
\label{subsec:thm}

We now state our main convergence result.

\begin{theorem}\label{thm:main}
    Suppose that Assumptions \ref{assump:noiseless}-\ref{assump:init} hold, $n=O(d \log^2 d)$, and $\epsilon_0 = O(1/\log(d)^4)$. Then, Algorithm \ref{alg:bwgd_ds} using the theoretical dynamic smoothing update $\epsilon_t = 2 \sqrt{\|\bu_t - \bu_\star\|/\|\bu_\star\|}$ achieves the error contraction
    \[
        \|\bu_t - \bu_\star \|^2 \leq \Big[\prod_{i=1}^t \Big(1-\frac{2}{\kappa_t + 1} \Big) \Big] \|\bu_0 - \bu_\star\|^2,
    \]
    for a sequence $\kappa_t \searrow 1$. In other words, $\bu_t$ converges to $\bu_\star$ superlinearly.
\end{theorem}

Several aspects of this result merit discussion:

\textbf{Sample complexity:} Our analysis requires $n = O(d\log^2 d)$ samples, which exceeds the information-theoretic limit of $n = O(d)$ achieved by gradient descent methods \cite{chen2015solving,ma2020implicit}. The additional logarithmic factors arise from our concentration arguments for controlling the smoothed Hessian. This matches the sample complexity of existing Gauss-Newton analyses for real signals \cite{gao2017phaseless}, while complex signal analysis requires $n = O(d\log^3 d)$ \cite{huang2024asymptotic}. As discussed earlier, it may be possible to get the best of both worlds (optimal sample complexity and superlinear convergence) by running a two-stage algorithm.

\textbf{Initialization requirement:} The required initialization accuracy of $O(1/\log^8(d))$ is more stringent than the $O(1)$ requirement in previous Gauss-Newton analyses \citep{gao2017phaseless,huang2024asymptotic}. This can be achieved by standard spectral initialization with an additional logarithmic factor in sample complexity, or by using spectral initialization with a few steps of gradient descent before running BWGD-DS, since gradient descent does not have as tight an initialization requirement.

\textbf{Convergence rate:} The superlinear convergence manifests through the decreasing sequence $\kappa_t \searrow 1$. Initially, when $\epsilon_t$ is large, the algorithm exhibits linear convergence with a modest contraction factor. As $\epsilon_t$ decreases and the iterates approach $\bu_\star$, the contraction factor approaches that of Newton's method, yielding superlinear convergence.

\subsection{Proof of Theorem \ref{thm:main}}
\label{subsec:proof}

The proof proceeds in three main steps: (1) establishing that $\bu_\star$ is a stationary point and local minimum of the smoothed objective for all $\epsilon$, (2) proving strong convexity and smoothness of $F_{\epsilon,w}$ in a neighborhood of $\bu_\star$, and (3) applying standard optimization theory to obtain the convergence rate. The proof of the first step is straightforward, and proof of the third step is well-known, and so we do not give their proofs here. We do give formal statements of these results as follows.

We begin by describing steps (1) and (3). First, for step (1), it is straightforward to show that $\bu_\star$ is a minimizer of $F_\epsilon$ for all $\epsilon$.
\begin{lemma}\label{lem:stationary}
    Under Assumption \ref{assump:noiseless}, $\bu_\star$ is a local minimum of $F_\epsilon(\bu)$ for all $\epsilon \geq 0$.
\end{lemma}
The proof of this follows from the fact that $F_\epsilon(\bu_\star) = 0$ for all $\epsilon$ and $F_\epsilon \geq 0$.

Step (3) of our proof relies on the following standard result from convex optimization \citep{bubeck2015convex}. We say that a twice differentiable function $f:\R^d \to \R$ is $\alpha$-strongly convex and $\beta$-smooth if $\alpha \bI \preceq \nabla^2 f(\bx) \preceq \beta \bI$. We state this result in a local manner to deal with the fact that in our setting we expect $\alpha,\beta \to 1$ as $\bu_t \to \bu_\star$.
\begin{lemma}\label{lem:gd}
    Suppose that $F$ is $\alpha$-strongly convex and $\beta$-smooth for all $\bu \in \{\bu \in \R^d : \|\bu - \bu_\star\| \leq \|\bu_t - \bu_\star\|  \}$, and $\kappa = \beta / \alpha$. Then, gradient descent with step size $1/\beta$ achieves the error contraction
    \begin{equation}
        \|\bu_{t+1} - \bu_\star\|^2 \leq \left(1-\frac{2}{\kappa + 1}\right) \|\bu_t - \bu_\star\|^2
    \end{equation}
\end{lemma}

Thus, the key technical challenge is establishing tight strong convexity and smoothness bounds to be able to argue that the sequence converges to $\bu_\star$ at a superlinear rate. In particular, we need to show that for all $\bu_t$ and $\epsilon_t$ in our sequence, the function $F_{w,\epsilon_t}$ remains strongly convex and smooth over the set $\{\bu \in \R^d : \|\bu - \bu_\star\| \leq \|\bu_t - \bu_\star\|  \}$. Furthermore, we must also show that as $\bu_t$ approaches $\bu_\star$ and $\epsilon_t$ approaches zero, the strong convexity and smoothness parameters approach 1.

We establish the crucial local strong convexity and smoothness properties in the following proposition.
\begin{proposition}\label{prop:scsm}
    Suppose that $n = O(d \log^2 d)$, $\|\bu - \bu_\star\| \leq \delta \|\bu_\star\|$ for $\delta=O(1/\log^4(d))$, and $\epsilon = 2\sqrt{\delta}$. Then, $F_{\epsilon,w}$ is $(1/32-2\epsilon)$-strongly convex and $(1+\epsilon)$-smooth at $\bu$ with high probability. Furthermore, with probability 1, there exists a ball of radius $r_\star$ around $\bu_\star$ such that $F_{\epsilon, w}$ is $(1-\epsilon)$-strongly convex within this ball.
\end{proposition}
The proof of Proposition \ref{prop:scsm} involves careful analysis of the Hessian of $F_{\epsilon,w}$, decomposing it into two terms. These terms are then controlled by the smoothing parameter $\epsilon$ and the proximity to $\bu_\star$, which we denote by $\delta$ following Assumption \ref{assump:init}. Note that this proposition shows that when $\bu_t$ is further from $\bu_\star$, the strong convexity constant may be small, but once it is sufficiently close, the parameter approaches 1.

We can now give the proof of our main theorem.

\begin{proof}[Proof of Theorem \ref{thm:main}]
Suppose that we initialize close enough so that the function $F_{w, \epsilon}$ is $(1/32-2\epsilon_0)$-strongly convex and $(1+\epsilon_0)$-smooth following Proposition \ref{prop:scsm}. Since $\bu_\star$ is stationary by Lemma \ref{lem:stationary}, then it is the unique stationary point in this neighborhood by strong convexity.

From Lemma \ref{lem:gd}, gradient descent with step size $1/(1+\epsilon_0)$ obtains the contraction
\[
     \|\bu_1 - \bu_\star\|^2 \leq (1-\frac{2}{\kappa_0+1}) \|\bu_0 - \bu_\star\|^2.
\]
Iterating this contraction, we eventually reach a neighborhood where $F_{\epsilon_t, w}$ is $1-\epsilon_t$ strongly convex and $1+\epsilon_t$ smooth by Proposition \ref{prop:scsm}. Thus, the proof of the Theorem follows by applying Lemma \ref{lem:gd} with the previous estimates and step size $1/(1+\epsilon_t)$.
\end{proof}

Note that once BWGD-DS reaches the local region where $F_{\epsilon_t,w}$ is $(1-\epsilon_t)$ smooth and $(1+\epsilon_t)$ convex, the contraction factor simplifies
\[
    \Big(1 - \frac{2}{\frac{1+\epsilon_t}{1-\epsilon_t} + 1} \Big) = \epsilon_t.
\]
Noting that $\epsilon_t = 2 \sqrt{\|\bu_t - \bu_\star\|}$ (assuming WLOG $\|\bu_\star\| = 1$), we get the recursion
\[
    \|\bu_{t+1} - \bu_\star\|^2 \leq 2 \|\bu_{t} - \bu_\star\|^{5/2},
\]
which thus yields an explicit superlinear convergence bound.

\subsection{Supplemental Proofs}
\label{subsec:supp_proofs}

In this section, we give the proof Proposition \ref{prop:scsm}. The proof combine concentration inequalities for Gaussian random matrices with covering arguments on the unit sphere, as well as a careful analysis of our smoothed objective.

We begin by first proving a supplemental proposition in Section \ref{subsubsec:good_subset}, and then proceed to the proof of Proposition \ref{prop:scsm} in Section \ref{subsubsec:scsm}.

\subsubsection{A Supplemental Proposition}
\label{subsubsec:good_subset}

Before proceeding to the proof of Proposition \ref{prop:scsm}, we first control the behavior of the sensing vectors near $\bu_\star$.
\begin{proposition}\label{prop:good_subset}
    Suppose $n=d \log^2 d$. Then, for all $\tau <1/(128\log(d))$ and for all $\bu \in S^{d-1}$, there exists a subset $\check \cA$ of $\cA$ with $\lceil(1-\tau)n\rceil$ points, such that with high probability:
    \begin{enumerate}
        \item $\frac{\tilde \ba_i^T}{\|\tilde \ba_i\|} \bu = \Omega(\tau) $ for all $\ba_i \in \check \cA$
        \item $\sigma_d(\tilde \bA_{i}) = O(\sqrt{d \log^2 d})$, implying $\frac{1}{n} \sum_{\check \cA(\bu)} \ba_i \ba_i^T = \Theta(\bI)$
        \item $\frac{1}{n} \sum_{\check \cA(\bu)} \bC_{\cA}^{-1/2}\ba_i \ba_i^T \bC_{\cA}^{-1/2} \succeq \bI/32$
    \end{enumerate}
\end{proposition}

This proposition ensures that for any direction $\bu$, most sensing vectors maintain a non-negligible inner product with $\bu$ while also being sufficiently well spread, so that their covariance is $\Omega(\bI)$. The proof uses covering arguments and concentration inequalities.

The proof of Proposition \ref{prop:good_subset} relies on four lemmas that control the behavior of the set of vectors $\cA = \{\ba_i: \ i=1, \dots, n\}$.

\textbf{Controlling the sample covariance of $\cA$:} Our first lemma is a corollary of \cite{vershynin2012close},  which bounds how close the sample covariance is to the population covariance for our sensing vectors.
\begin{lemma}\label{lem:cov_conc}
    With probability at least $1-p$, $\|\bC_{\cA} - \bI \|_{2} \lesssim_{p} \sqrt{\frac{d}{n}}$
\end{lemma}
In particular, once $n = O(d \log d)$, we see that the error tends to zero as $d \to \infty$. 

\textbf{Controlling the norms of $\cA$:} Our second lemma uses concentration of the sample covariance and $\chi^2$ random variables to bound the norm of the sensing vectors $\ba_i$ and the whitened sensing vectors $\tilde \ba_i$.
\begin{lemma}\label{lem:norma}
     Once $n=O(d \log d)$, $\| \ba_i\|^2, \|\tilde \ba_i\|^2 \in  (1 \pm c_1) d$ with probability at least $1-O(d\log d \exp(-cd))$.
\end{lemma}

\textbf{Controlling the directions in $\cA$:} In our third lemma, we control how many $\tilde \ba_i$ are close to being orthogonal to any $\bu$. To do this, we first prove a lemma that controls the unwhitened vectors, and then extend this below to the whitened vectors in Lemma \ref{lem:bern_conc_whiten}.
\begin{lemma}\label{lem:bern_conc}
    Let $\ba_i \sim N(\bzero, \bI)$. Then,
    \begin{align}
        \Pr&\Big(\sum \bbone\Big(\frac{\ba_i^T}{\|\ba_i\|} \bu \leq \tau\Big)\leq 6n\tau/\pi \text{ for all } \bu \in S^{d-1}\Big) \\ \nonumber
        &\geq 1- 2(1+2/\tau)^d \exp\left(- \frac{2 n \tau^2}{\pi}\right).
    \end{align}
\end{lemma}
\begin{proof}
     
Fix a $\bu \in S^{d-1}$. The probability that $\frac{\ba_i^T}{\|\ba_i\|} \bu \leq \tau$ is $\tau/\pi$. Using Bernoulli concentration (Hoeffding's inequality)
\[
    \Pr( \sum \bbone\Big( \frac{\ba_i^T}{\|\ba_i\|} \bu \leq \tau \Big) \leq 2n\tau/\pi) \geq 1-   \exp\left(- \frac{2 n \tau^2}{\pi^2} \right)
\]
and
\[
    \Pr( \sum \bbone\Big( \frac{\ba_i^T}{\|\ba_i\|} \bu \leq 2\tau \Big) \leq 4n\tau/\pi) \geq 1-  \exp\left(- \frac{8 n \tau^2}{\pi^2} \right)
\]
Thus, putting these together,
\begin{align*}
    \Pr&\Big(\sum \bbone\Big(\frac{\ba_i^T}{\|\ba_i\|} \bu' \leq \tau\Big)\leq 6n\tau/\pi \text{ for all } \bu' \in B(\bu, \tau)\Big) \\
    &\geq 1-  2\exp\left(- \frac{2 n \tau^2}{\pi}\right)
\end{align*}

The sphere $S^{d-1}$ can be covered by $(1+2/\tau)^d$ balls of radius $\tau$. Thus, by a covering argument,
\begin{align*}
    \Pr&\Big(\sum \bbone\Big(\frac{\ba_i^T}{\|\ba_i\|} \bu \leq \tau\Big)\leq 6n\tau/\pi \text{ for all } \bu \in S^{d-1}\Big) \\
    &\geq 1- 2(1+2/\tau)^d \exp\left(- \frac{2 n \tau^2}{\pi}\right).
\end{align*}

\end{proof}
We extend this to the whitened vectors $\tilde \ba_i$ in the following lemma.
\begin{lemma}\label{lem:bern_conc_whiten}
    Let $\ba_i \sim N(\bzero, \bI)$ and let $\tilde \ba_i = \bSigma_{\cA}^{-1/2} \ba_i$. Further suppose that $ n = O\left( \frac{d\log(1/\tau)}{\tau^2} \right).$ Then, conditioned on $\|\ba_i\|, \|\tilde \ba_i\|\in (1\pm c_1)d$ 
    \begin{align}
        \Pr&\Big(\sum \bbone\Big(\frac{\tilde \ba_i^T}{\| \tilde \ba_i\|} \bu \leq \tau\Big)\leq \tilde O(\tau) \text{ for all } \bu \in S^{d-1}\Big) \\ \nonumber
        &\geq 1- 2(1+2/\tau)^d \exp\left(- \frac{2 n \tau^2}{\pi}\right).
    \end{align}
\end{lemma}
\begin{proof}
Conditioned on $\|\ba_i\|, \|\tilde \ba_i\|\in (1\pm c_1)d$, for $\bu \in S^{d-1}$,
\begin{align*}
       \Big|\frac{\ba_i^T \bu}{\|\ba_i\|} - \frac{\tilde \ba_i \bu}{\|\tilde \ba_i\|}  \Big|\leq \|\frac{\ba_i}{\|\ba_i\|} - \frac{\tilde \ba_i}{\|\tilde \ba_i\|}\|= O(\|\bC_{\cA} - \bI\|_2)
\end{align*}
with high probability. Also,
\begin{align*}
    &\Big|\sum \bbone\Big(\frac{\ba_i^T}{\|\ba_i\|} \bu \leq \tau\Big) - \sum \bbone\Big(\frac{\tilde \ba_i^T}{\|\tilde \ba_i\|} \bu \leq \tau\Big) \Big| \leq \\
    &\sum \bbone\Big(\frac{\ba_i^T}{\|\ba_i\|} \bu \leq \tau+O(\|\bC_{\cA} - \bI\|_2)\Big)
\end{align*}
In particular, for $n = O\left( \frac{d\log(1/\tau)}{\tau^2} \right)$, $\|\bSigma_\cA - \bI\|_2 = O(\tau)$, and so
\begin{align*}
    &\Big|\sum \bbone\Big(\frac{\ba_i^T}{\|\ba_i\|} \bu \leq \tau\Big) - \sum \bbone\Big(\frac{\tilde \ba_i^T}{\|\tilde \ba_i\|} \bu \leq \tau\Big) \Big| \leq \\
    &\sum \bbone\Big(\frac{\ba_i^T}{\|\ba_i\|} \bu \leq  O(\tau)\Big) =  O(\tau).
\end{align*}
\end{proof}

In the previous lemma, we see an $1-O(\tau)$ fraction of points $\tilde \ba_i$ have angle with $\bu$ less than $\tilde O(\pi/2 - \tau)$ with high probability when
\[
    n = O\left( \frac{d\log(1/\tau)}{\tau^2} \right).
\]

\textbf{Controlling the spread of $\cA$:} Finally, let $\bA$ and $\tilde \bA$ be the matrices obtained by collecting $\ba_i$ and $\tilde \ba_i$ as columns. We can bound the spectral norm of submatrices of $\bA$, and consequently of $\tilde \bA$. We note that the constants are not optimized. 

\begin{lemma}\label{lem:subsetsingval}
    Let $\bA_i$ be an enumeration of submatrices of $\bA$ of size $d \times m$, for $m \geq (1-\gamma)n$,  where $n = d \log^2 d$ and $\gamma < 1/(128 \log d)$. Then,
    \begin{align*}
    \Pr&\Big(\sigma_d(\bA_i) \geq \sqrt{d\log^2 d / 32}, \ \text{ for all submatrices } \bA_i\Big) \\
    &\geq 1- 2\exp(-\frac{d \log^2 d}{128}).
\end{align*}
    In particular, with high probability for all such submatrices, $\frac{1}{n} \bA_i \bA_i^T \succeq c_\tau \bI$ for some constant $c_\tau\geq 1/32$.
\end{lemma}

\begin{proof}

Let $\bA$ be a $d \times n$ matrix whose columns are $\ba_i \overset{i.i.d.}{\sim} N(\bzero, \bI)$. For any submatrix of $m> d$ columns $\bA_i$, we have
\begin{equation}
    \Pr(\sigma_d(\bA_i) \geq \sqrt{m}-\sqrt{d} - t) \geq 1- 2\exp(-t^2/2).
\end{equation}
In particular, 
\begin{align*}
    \Pr&\Big(\sigma_d(\bA_i) \geq \sqrt{m}-\sqrt{d} - t, \ \forall \ \bA_i\Big) \\
    &\geq 1- 2{n \choose m}\exp(-t^2/2)\\
    &\geq 1- 2\exp(-t^2/2 + (n-m)\log(n)).
\end{align*}
Using $\sqrt{x}-\sqrt{y} \geq \sqrt{x/2-y}$ for $x \geq 2y$, if $m \geq 2d$, then
\begin{align*}
    \Pr&\Big(\sigma_d(\bA_i) \geq \sqrt{m/2-d} - t, \ \forall \ \bA_i\Big) \\
    &\geq 1- 2\exp(-t^2/2 + (n-m)\log(n)).
\end{align*}
Let $m = (1-\gamma)n$. Then,
\begin{align*}
    \Pr&\Big(\sigma_d(\bA_i) \geq \sqrt{(1-\gamma)n/2-d} - t, \ \forall \ \bA_i\Big) \\
    &\geq 1- 2\exp(-t^2/2 + \gamma n\log(n)).
\end{align*}

    Now let $n = d \log^2 d$. For $(1-\gamma)/2 < 1/4$ and $\log d \geq 3$, we have
\begin{align*}
    \Pr&\Big(\sigma_d(\bA_i) \geq \sqrt{d \log^2 d/8} - t, \ \forall \ \bA_i\Big) \\
    &\geq 1- 2\exp(-t^2/2 + \gamma d \log^3 d).
\end{align*}
Letting $t = \frac{1}{2} \sqrt{d \log^2 d / 8}$,
\begin{align*}
    \Pr&\Big(\sigma_d(\bA_i) \geq \sqrt{d\log^2 d / 32}, \ \forall \ \bA_i\Big) \\
    &\geq 1- 2\exp(-\frac{d \log^2 d}{64} + \gamma d \log^3 d).
\end{align*}
Thus, for $n = d \log^2 d$, we see that $\sigma_d(\bA_i) = \Omega(\sqrt{d \log^2 d})$ for all subsets of size $m = (1-\gamma) d \log^2 d$ provided that $\log d \geq 3$ and $\gamma \leq 1/(128 \log d)$. Thus, the lemma follows.
\end{proof}

\begin{proof}[Proof of Proposition \ref{prop:good_subset}]
    The proof of Proposition \ref{prop:good_subset} follows from Lemmas \ref{lem:cov_conc}, \ref{lem:norma}, \ref{lem:bern_conc_whiten},  and \ref{lem:subsetsingval}. 
\end{proof}

\subsubsection{Proof of Proposition \ref{prop:scsm}}
\label{subsubsec:scsm}

We now proceed with the proof of Proposition \ref{prop:scsm}.

\begin{proof}[Proof of Proposition~\ref{prop:scsm}]
\begin{equation}
    F_{\epsilon,w}(\bu) = \frac{1}{2n} \sum_{i=1}^n \left(\sqrt{(\tilde \ba_i^T \bu)^2 + \epsilon \|\tilde \ba_i\|^2} - \sqrt{y_i^2 + \epsilon \|\tilde \ba_i\|^2}\right)^2
\end{equation}
Let $\bu_t = (1-t) \bu_0 + t \bu_1$.
Let $r_i(\bu) = \left(\sqrt{(\tilde \ba_i^T \bu)^2 + \epsilon \|\tilde \ba_i\|^2} - \sqrt{y_i^2 + \epsilon \|\tilde \ba_i\|^2}\right)$. We compute the derivative
\begin{align*}
    \partial_t F_{\epsilon,w}(\bu_t)|_{t=0} &= \frac{1}{n} \sum_i r_i(\bu_0) \frac{\tilde \ba_i^T \bu_0 \dot \bu_0^T \tilde \ba_i }{\sqrt{\tilde \ba_i^T \bu_t0 \bu_0^T \tilde \ba_i + \epsilon \|\tilde \ba_i\|^2}} \\
    &= \Big \langle \frac{1}{n} \sum_i r_i(\bu_0) \frac{\tilde \ba_i \tilde \ba_i^T \bu_0   }{\sqrt{\tilde \ba_i^T \bu_t \bu_0^T \tilde \ba_i + \epsilon \|\tilde \ba_i\|^2}}, \dot \bu_0^T \Big \rangle.
\end{align*}

We compute and simplify the second derivative
\begin{align*}
    \partial_t^2 &F_{\epsilon,w}(\bu_t) =  \frac{1}{n} \sum_i   \cancel{\frac{(\tilde \ba_i^T \bu_t \dot \bu_t^T \tilde \ba_i)^2 }{\tilde \ba_i^T \bu_t \bu_t^T \tilde \ba_i + \epsilon \|\tilde \ba_i\|^2}}\\
    &-  \Big(\cancel{\sqrt{\tilde \ba_i^T \bu_t \bu_t^T \tilde \ba_i + \epsilon \|\tilde \ba_i\|^2}} \\
    &\ \ - \sqrt{y_i^2 + \epsilon \|\tilde \ba_i\|^2}\Big) \frac{(\tilde \ba_i^T \bu_t \dot \bu_t^T \tilde \ba_i)^2 }{(\tilde \ba_i^T \bu_t \bu_t^T \tilde \ba_i + \epsilon \|\tilde \ba_i\|^2)^{3/2}} \\
    & + r_i(\bu_t) \frac{\tilde \ba_i^T \dot \bu_t \dot \bu_t^T \tilde \ba_i }{\sqrt{\tilde \ba_i^T \bu_t \bu_t^T \tilde \ba_i + \epsilon \|\tilde \ba_i\|^2}} \\
    &= \underbrace{\frac{1}{n} \sum_i   \left(  \sqrt{y_i^2 + \epsilon \|\tilde \ba_i\|^2}\right) \frac{(\tilde \ba_i^T \bu_t \dot \bu_t^T \tilde \ba_i)^2 }{(\tilde \ba_i^T \bu_t \bu_t^T \tilde \ba_i + \epsilon \|\tilde \ba_i\|^2)^{3/2}}}_{\mathrm{(I)}} \\
    & + \underbrace{\frac{1}{n} \sum_i\left(1-\frac{\sqrt{y_i^2 + \epsilon \|\tilde \ba_i\|^2} }{\sqrt{\tilde \ba_i^T \bu_t \bu_t^T \tilde \ba_i + \epsilon \|\tilde \ba_i\|^2}}\right) \tilde \ba_i^T \dot \bu_t \dot \bu_t^T \tilde \ba_i  }_{\mathrm{(II)}}
\end{align*}
The second derivative is divided into two terms. The first term is nonnegative, while the second term may be positive or negative. We will bound each of these sums separately at $t=0$. In essence, we will show that for well-chosen $\epsilon$ depending on $\delta$, the first term approaches $\|\dot \bu_0\|$ while the second term goes to zero as $\delta \to 0$.

Before we bound these terms, we clarify our setting. First, we condition on the event that $\|\tilde \ba_i\|^2 \in (1 \pm c_1) d$ for a small $c_1$, which occurs with high probability by Lemma \ref{lem:norma}. Also, we assume without loss of generality that $\|\bu_\star\|=1$. In the case where the norm is not 1, all results translate by rescaling. Finally, under the assumption that $\|\bu_0 - \bu_\star\| \leq \delta \|\bu_\star\| = \delta$, we can upper and lower bound $|\bu_0^T \tilde \ba_i|^2$:
\begin{align}\label{eq:u0lb}
    |\bu_0^T \tilde \ba_i|^2 &=  |(\bu_0^T -  \bu_\star) \tilde \ba_i + \bu_\star \tilde \ba_i|^2 \\ \nonumber &\geq |\bu_\star \tilde \ba_i|^2 -2\delta |\bu_\star \tilde \ba_i|  \|\tilde \ba_i\| \\ \nonumber &
    \geq y_i^2 - 2\delta \|\tilde \ba_i\|^2
\end{align}
and
\begin{align}\label{eq:u0ub}
    |\bu_0^T \tilde \ba_i|^2 
    \leq  |\bu_\star \tilde \ba_i|^2 + (2\delta+\delta^2) \|\tilde \ba_i\|^2.
\end{align}
These bounds will be used in our upper and lower bounds for terms (I) and (II).

\noindent
\textbf{Bound for (II):} We proceed first with bounding the second term, since it is simpler. Using \eqref{eq:u0lb}, we can upper bound the ratio
\begin{align*}
    \frac{\sqrt{y_i^2 + \epsilon \|\tilde \ba_i\|^2}}{\sqrt{ |\bu_0^T \tilde \ba_i|^2 + \epsilon \|\tilde \ba_i\|^2}} 
    \leq \sqrt{\frac{\epsilon}{\epsilon-2\delta}}.
\end{align*}
Similarly, using \eqref{eq:u0ub}, we can lower bound the ratio
\begin{align*}
    \frac{\sqrt{y_i^2 + \epsilon \|\tilde \ba_i\|^2}}{\sqrt{\tilde \ba_i^T \bu_0 \bu_0^T \tilde \ba_i + \epsilon \|\tilde \ba_i\|^2}} 
    \geq \sqrt{\frac{\epsilon}{\epsilon+3\delta}}.
\end{align*}
These follow from the fact that these functions are monotonic with respect to $y_i^2$ and taking appropriate limits.

Thus,
\begin{align}\label{eq:term_ii_lb}
    &\frac{1}{n} \sum_i \left(1-\frac{\sqrt{y_i^2 + \epsilon \|\tilde \ba_i\|^2} }{\sqrt{\tilde \ba_i^T \bu_0 \bu_0^T \tilde \ba_i + \epsilon \|\tilde \ba_i\|^2}}\right) \tilde \ba_i^T \dot \bu_0 \dot \bu_0^T \tilde \ba_i \geq \\ \nonumber
    & (1- \sqrt{\frac{\epsilon}{\epsilon-2\delta}})\|\dot\bu_0\|^2   
\end{align}
and
\begin{align} \label{eq:term_ii_ub}
    &\frac{1}{n} \sum_i \left(1-\frac{\sqrt{y_i^2 + \epsilon \|\tilde \ba_i\|^2} }{\sqrt{\tilde \ba_i^T \bu_0 \bu_0^T \tilde \ba_i + \epsilon \|\tilde \ba_i\|^2}}\right) \tilde \ba_i^T \dot \bu_0 \dot \bu_0^T \tilde \ba_i \leq \\ \nonumber
    & (1- \sqrt{\frac{\epsilon}{\epsilon+3\delta}})\|\dot\bu_0\|^2   
\end{align}
Thus, choosing $\delta = o(\epsilon)$, we see that this term vanishes as $\delta \to 0$.

\noindent
\textbf{Bound for (I):} We begin our lower bound of the first term by applying \eqref{eq:u0lb} 
\begin{align*}
    \frac{1}{n} \sum_i &\left(  \sqrt{y_i^2 + \epsilon \|\tilde \ba_i\|^2}\right) \frac{(\tilde \ba_i^T \bu_0 \dot \bu_0^T \tilde \ba_i)^2 }{(\tilde \ba_i^T \bu_0 \bu_0^T \tilde \ba_i + \epsilon \|\tilde \ba_i\|^2)^{3/2}} \geq \\
    &\sqrt{\frac{\epsilon}{\epsilon+3\delta}} \frac{1}{n} \sum_i \frac{(\tilde \ba_i^T \bu_0)^2 }{\tilde \ba_i^T \bu_0 \bu_0^T \tilde \ba_i + \epsilon \|\tilde \ba_i\|^2} ( \dot \bu_0^T \tilde \ba_i)^2
\end{align*}
By Proposition \ref{prop:good_subset}, for a small $\tau$, there is a subset $\check \cA(\bu_0)$ of $\check \cA$ such that with high probability, $\frac{\tilde \ba_i^T}{\|\tilde \ba_i\|} \bu_0 > \tau $ for all $\ba_i \in \check \cA$, and $\frac{1}{n} \sum_{\check \cA(\bu)} \bC_{\cA}^{-1/2}\ba_i \ba_i^T \bC_{\cA}^{-1/2} = \Theta(\bI)$.

Thus, we see that for any choice of $\bu_0$, $\dot \bu_0$,
\begin{align}\label{eq:term_i_lb}
    \frac{1}{n} \sum_i &\frac{(\tilde \ba_i^T \bu_0)^2 }{\tilde \ba_i^T \bu_0 \bu_0^T \tilde \ba_i + \epsilon \|\tilde \ba_i\|^2} ( \dot \bu_0^T \tilde \ba_i)^2 \\
    &= \frac{1}{n} \sum_{\check{\cA}(\bu_0)} \frac{(\tilde \ba_i^T \bu_0)^2 /\|\tilde \ba_i\|^2}{\tilde \ba_i^T \bu_0 \bu_0^T \tilde \ba_i / \|\tilde \ba_i\|^2 + \epsilon } ( \dot \bu_0^T \tilde \ba_i)^2 \\ \nonumber
    &\geq \frac{1}{n} \sum_{\check{\cA}(\bu_0)} \frac{\tau^2 }{\tau^2 + \epsilon} ( \dot \bu_0^T \tilde \ba_i)^2 \\ \nonumber
    &\geq \frac{1}{32}\frac{\tau^2 }{\tau^2 + \epsilon}  \|\dot \bu_0\|^2
\end{align}
Following Lemma

On the other hand, it is straightforward to show that
\begin{align}\label{eq:term_i_ub}
    \frac{1}{n} \sum_i &\frac{(\tilde \ba_i^T \bu_0)^2 }{\tilde \ba_i^T \bu_0 \bu_0^T \tilde \ba_i + \epsilon \|\tilde \ba_i\|^2} ( \dot \bu_0^T \tilde \ba_i)^2 \\ \nonumber
    &\leq \frac{1}{n} \sum_{i}  ( \dot \bu_0^T \tilde \ba_i)^2 \\ \nonumber
    &= \|\dot \bu_0\|^2
\end{align}

Suppose that we set $2^{1/4}\delta^{1/8} = \epsilon^{1/4} = \tau < 1/(128 \log d) $. Then one can show that
\[
    (2- \sqrt{\frac{\epsilon}{\epsilon+3\delta}})\leq 1+\epsilon, 1-\sqrt{\frac{\epsilon}{\epsilon-2\delta}} \geq -\epsilon.
\]
Furthermore choose $\tau^2 = \sqrt{\epsilon}$, we see that
\[
\frac{\tau^2 }{\tau^2 + \epsilon}\frac{1}{32}  \geq \frac{1}{32}(1-\sqrt{\epsilon}).
\]
Then, combining \eqref{eq:term_ii_ub} and \eqref{eq:term_i_ub}
\begin{align*}
    \partial_t^2 &F_{\epsilon,w}(\bu_t)|_{t=0} \leq (1+\epsilon) \|\dot \bu_0\|^2 
\end{align*}
On the other hand, combining \eqref{eq:term_ii_lb} and \eqref{eq:term_i_lb}, and using 
\begin{align*}
    \partial_t^2 &F_{\epsilon,w}(\bu_t)|_{t=0} \geq \Big(\frac{\tau^2 }{\tau^2 + \epsilon} c_\tau - (1- \sqrt{\frac{\epsilon}{\epsilon-2\delta}}) \Big] \|\dot \bu_0\|^2 \\
    &\geq (c_\tau(1-\sqrt{\epsilon})-\epsilon) \|\dot \bu_0\|^2.
\end{align*}
This implies that $F_{\epsilon,w}$ is $1+\epsilon$ smooth and $\frac{1}{32} - 2\sqrt{\epsilon} $ strongly convex. 

To finish the proof, we rely on the following lemma, which we state without proof. It essentially says that no sensing vectors are orthogonal to $\bu_\star$. This follows from the fact that the set of vectors orthogonal to $\bu_\star$ is measure zero. 
\begin{lemma}\label{lem:local_ball}
    With probability 1, there exists $r_\star > 0$ such that $\frac{\tilde \ba_i^T}{\|\tilde \ba_i\|}\bu_\star \geq 2 r_\star$ for all $i$. Consequently, for all $\bu \in B(\bu_\star, r_\star)$, we have $\frac{\tilde \ba_i^T}{\|\tilde \ba_i\|} \bu \geq r_\star$.
\end{lemma}

Using this, once $\|\bu_0 - \bu_\star\| \leq r_\star$, we get that $\frac{\tilde \ba_i^T}{\|\tilde \ba_i}\bu_0 \geq r_\star$ for all $i$. Repeating the same argument as before, we see that $F_{\epsilon,w}$ is still $1+\epsilon$ smooth but now it is $1-2\epsilon$ strongly convex. In particular, locally $F_{\epsilon,w}$ has strong convex and smoothness that tend to 1 as $\epsilon=2\sqrt{\delta} \to 0$
\end{proof}

\section{Numerical Experiments}

We evaluate the performance of our proposed algorithm on synthetic phase retrieval problems. All experiments are conducted in dimension $d = 200$. The true signal $\bu_\star$ is set to the constant unit-norm vector ${\bone}/{\sqrt{d}}$. The sensing vectors ${\ba_i}$ are sampled i.i.d. from the standard normal distribution. Unless otherwise stated, we average results over 100 independent trials.

We compare three methods:
\begin{itemize}
    \item BWGD-DS (Loss): BWGD-DS with the loss heuristic from \eqref{eq:eps_update_theoretical};
    \item BWGD-DS (Quantile): BWGD-DS with the quantile heuristic from \eqref{eq:eps_update_empirical};
    \item BWGD: Regular Bures-Wasserstein gradient descent with no smoothing.
\end{itemize}

Figure~\ref{fig:slice} compares the convergence rate of BWGD and BWGD-DS with our two heuristics. For all methods, we use the spectral initialization of \cite{gao2017phaseless}, and we plot the error versus iteration for a single generated dataset where $n=650$. We observe that both BWGD-DS (Quantile) and BWGD converge rapidly within a few iterations, exhibiting clear superlinear convergence in line with what one expects from Newton's method. BWGD-DS (Loss) also converges superlinearly, though it does not converge rapidly until the method reaches a local neighborhood.

\begin{figure}[h]
    \centering
    \includegraphics[width=0.7\columnwidth]{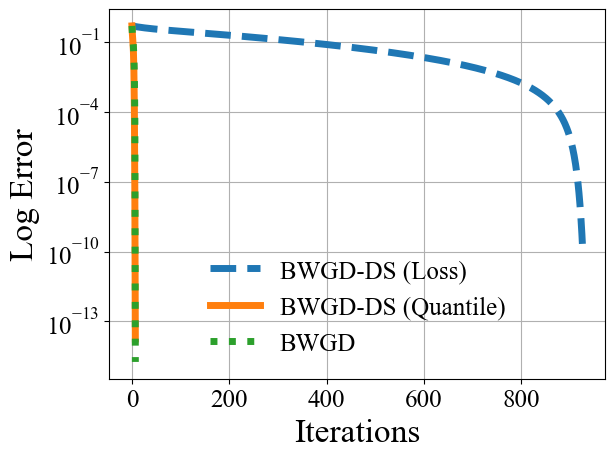}
    \caption{Log error $\log(|\bu_t - \bu_\star|)$ vs. iteration at $n = 650$, for a single trial. All methods initialized with spectral initialization}
    \label{fig:slice}
\end{figure}

The second experiment in Figure~\ref{fig:iterations} depicts how the loss heuristic update requires substantially more iterations to converge compared to the other two methods. This is consistent with its update method, which initially uses larger smoothing parameters, slowing early progress. In contrast, when successful, both BWGD and BWGD-DS (Quantile) exhibit much faster convergence, and their iteration counts are nearly indistinguishable at higher sample sizes. We see that all methods converge faster as the sample size grows.

\begin{figure}[h]
\centering
\includegraphics[width=0.7\columnwidth]{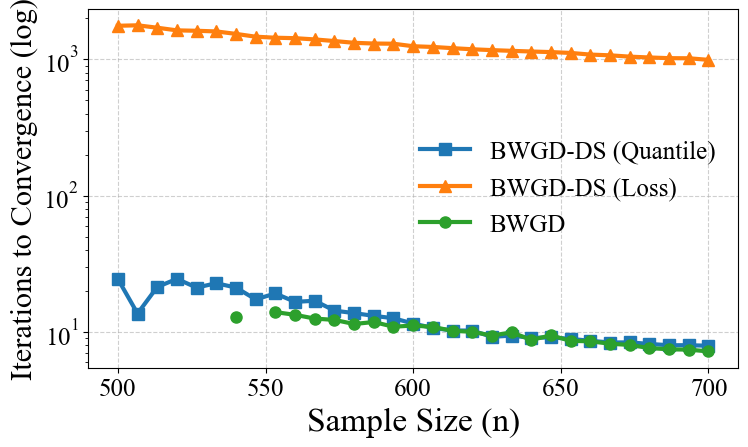}
\caption{Average number of iterations until convergence vs. number of samples $n$, using spectral initialization. Only successful trials are included.}
\label{fig:iterations}
\end{figure}

The third experiment in Figure~\ref{fig:success} examines the stability of BWGD-DS by examining how often it successfully converges for various sample sizes. A trial is deemed to have converged if the error, defined as $\|\bu_t - \bu_\star\|$, is less than $10^{-9}$. In the upper plot, we look at success rate versus sample size for spectral initialization, and in the bottom plot, we look at success rate versus sample size for random initialization. In both figures, we see that BWGD-DS with the loss heuristic achieves higher success rates when compared to other methods for smaller sample sizes. In other words, the loss heuristic has greater stability in more challenging examples with low sample sizes. 

In the upper plot of Figure \ref{fig:success}, where we use spectral initialization, BWGD and BWGD-DS (Quantile) only begin to converge reliably at around $n = 500$, at which point the theoretical approach already achieves near-perfect success rates. BWGD-DS (Quantile) still exhibits greater stability than BWGD in both cases. We note that, despite its greater stability, BWGD-DS (Loss) requires many more iterations to converge in general (greater than 1000 for BWGD-DS (Loss) versus less than 100 for BWGD-DS (Quantile)). In the lower plot of Figure~\ref{fig:success}, where we use random initialization, all methods now require significantly more samples to succeed. BWGD-DS (Loss) remains the most effective, achieving successful recovery at a lower threshold than the other methods. 

\begin{figure}[h]
\centering
\includegraphics[width=0.7\columnwidth]{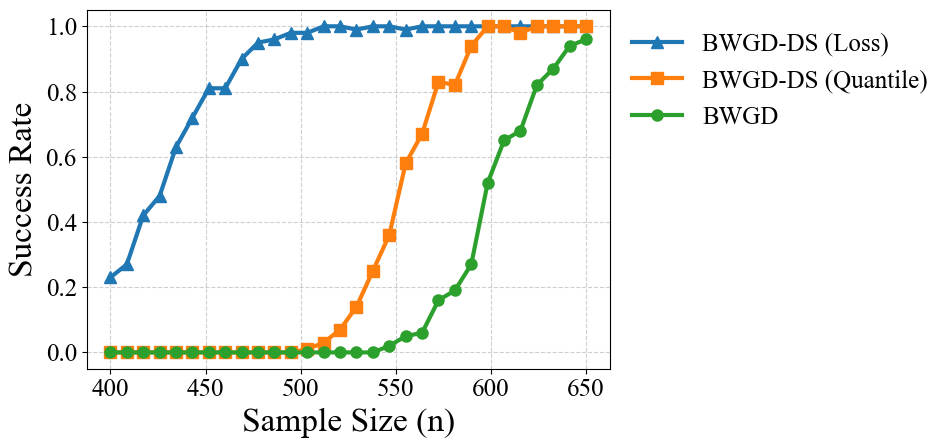}
\includegraphics[width=0.7\columnwidth]{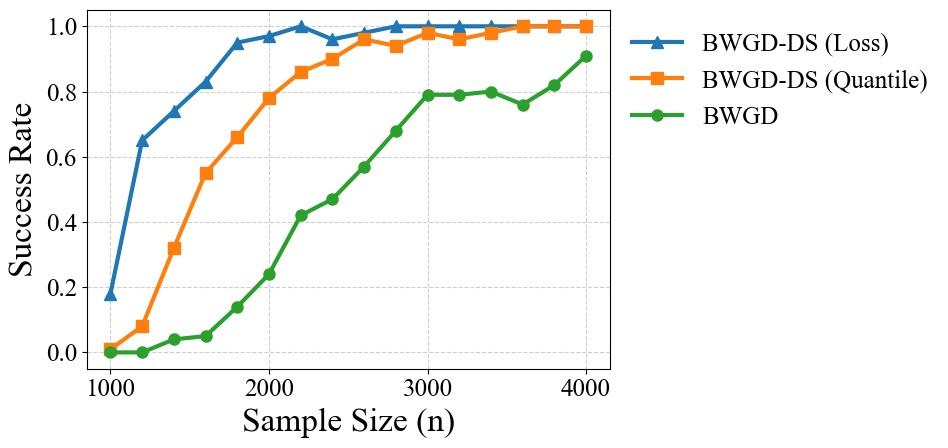}
\caption{Success rate vs. number of samples for the three BWGD variants. Success is defined if the method converges to a point with error less than $10^{-9}$. Top: methods run with spectral initialization. Bottom: methods run with random initialization. We see that the BWGD-DS (Loss) achieves the greatest stability, followed by BWGD-DS (Quantile) and then BWGD. This is at the cost of BWGD-DS (Loss) requiring significantly more iterations (greater than 1000 for BWGD-DS (Loss) versus less than 100 for BWGD-DS (Quantile) and BWGD).}
\label{fig:success}
\end{figure}

Figure~\ref{fig:convergence_time_grid} reports heatmaps of the log-error of each method across different sample sizes and iterations. Here, we display the average log error (or, log geometric mean) of the errors over 100 trials for each method. For both initialization methods, the results are similar to what we showed before. BWGD-DS (Loss) achieves superlinear convergence even at small sample sizes—earlier than the other methods, albeit at the cost of requiring significantly more iterations. The BWGD-DS (Quantile) has a rapid decrease in error once the sample size exceeds a threshold. Without dynamic smoothing, BWGD is fast when it converges, especially with spectral initialization. However, it struggles to succeed at lower sample sizes, particularly with random initialization.

\begin{figure}[htbp]
  \centering
  \begin{minipage}{\textwidth}
    \begin{minipage}{0.24\textwidth}
      \centering
      \text{Spectral}
    \end{minipage}
    \begin{minipage}{0.24\textwidth}
      \centering
      \text{Random}
    \end{minipage}
  \end{minipage}

  \begin{subfigure}[t]{0.24\textwidth}
    \centering
    \includegraphics[width=\linewidth]{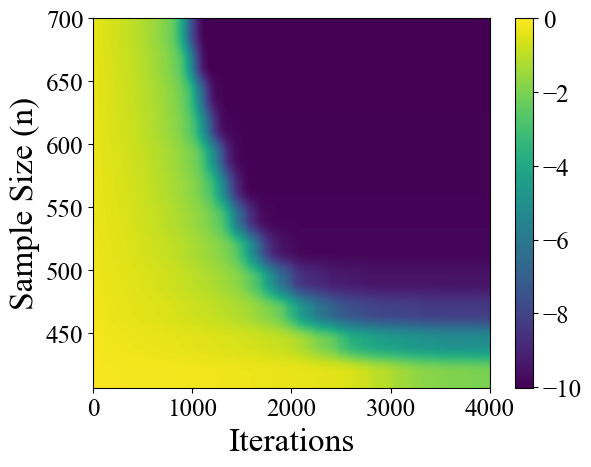}
  \end{subfigure}
  \hfill
  \begin{subfigure}[t]{0.24\textwidth}
    \centering
    \includegraphics[width=\linewidth]{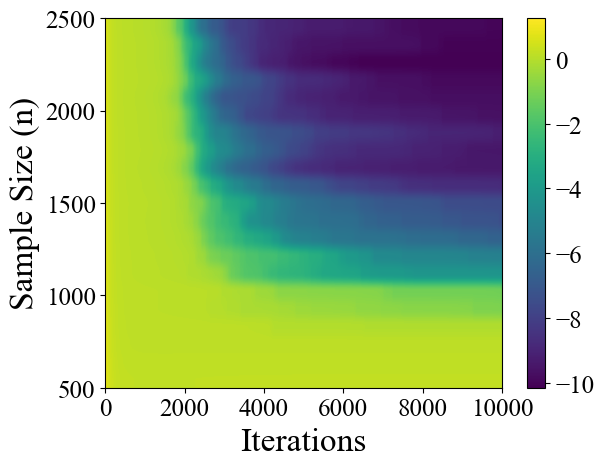}
  \end{subfigure}

  \begin{subfigure}[t]{0.24\textwidth}
    \centering
    \includegraphics[width=\linewidth]{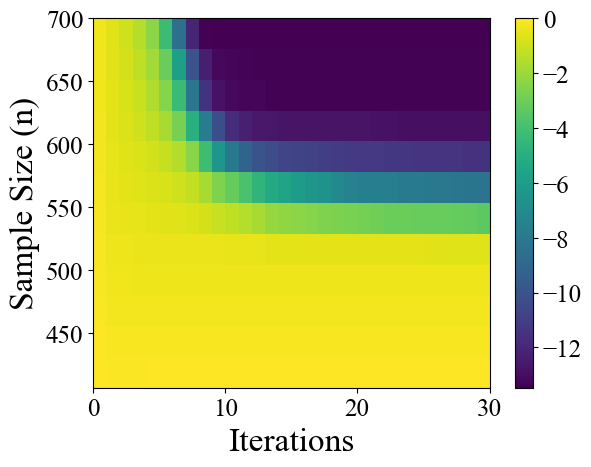}
  \end{subfigure}
  \hfill
  \begin{subfigure}[t]{0.24\textwidth}
    \centering
    \includegraphics[width=\linewidth]{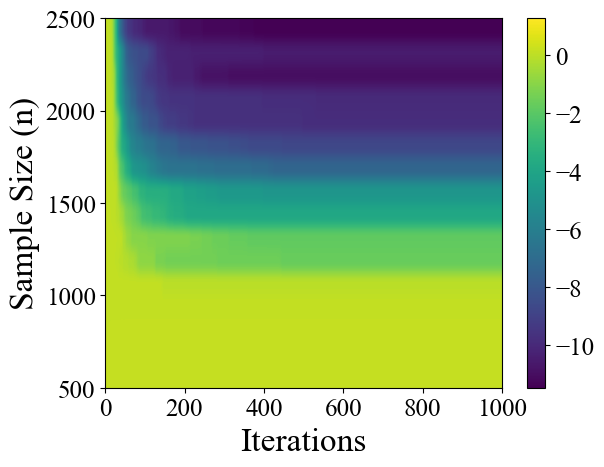}
  \end{subfigure}

  \begin{subfigure}[t]{0.24\textwidth}
    \centering
    \includegraphics[width=\linewidth]{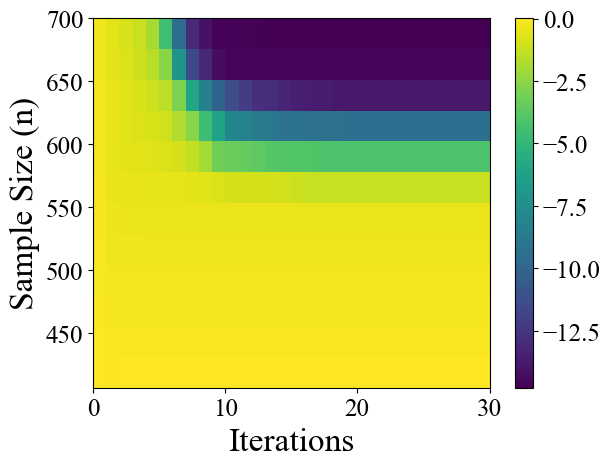}
  \end{subfigure}
  \hfill
  \begin{subfigure}[t]{0.24\textwidth}
    \centering
    \includegraphics[width=\linewidth]{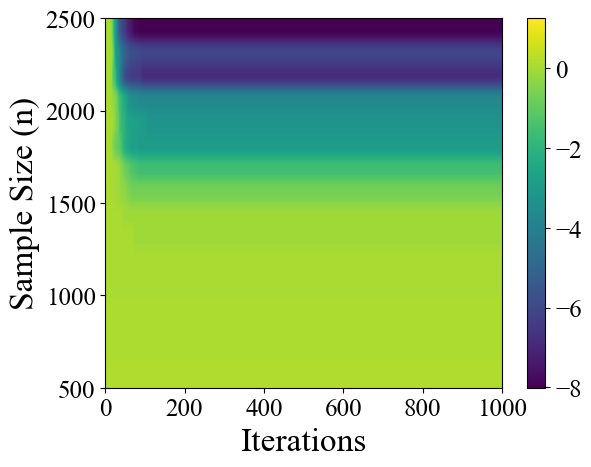}
  \end{subfigure}
  \caption{Heatmaps depicting the log of geometric mean error across 100 trials for sample size versus iteration. The left column uses spectral initialization and the right column uses random initialization. The top row is BWGD-DS (Loss), the middle row is BWGD-DS (Quantile), and bottom row is BWGD respectively. BWGD-DS (Loss) converges significantly slow than the other two methods, but is able to converge for slightly smaller sample sizes. BWGD-DS (Quantile) strikes a balance between speed and stability for small sample sizes when compared to BWGD.}
  \label{fig:convergence_time_grid}
\end{figure}

\section{Conclusion}

We have presented an analysis of BWGD for real phase retrieval. We have demonstrated that BWGD is indeed a nonsmooth Newton method for this problem. To stabilize this method, we introduce a dynamic smoothing procedure that achieves superlinear convergence. Experimental results demonstrate the superior stability of our regularized method compared to BWGD on the nonsmooth objective. In particular, the BWGD-DS with the loss heuristic exihibits the highest stability but slowest convergence, while BWGD-DS with the quantile heuristics demonstrates increased stability over BWGD without sacrificing on convergence rate.

For future work, it would be of interest to tighten the results here and prove optimal ways of setting the sequence of regularization parameters $\epsilon_t$. Another way to tighten the results here would be to consider a two-stage algorithm that combines the benefits of gradient descent and Newton's method. It would be interesting to extend these ideas and to develop stable Newton methods for complex and higher-rank nonsmooth matrix recovery problems. Finally, our method is reminiscent of quasi-Newton methods, where we construct a sequence of steps that approach the Newton step asymptotically. It would be interesting, therefore, to consider new smoothing approaches for quasi-Newton methods as alternatives to more classical approaches \citep{chen1997superlinear}.

\section{Acknowledgements}

T. Maunu was supported by the NSF under Award No. 2305315.

\bibliographystyle{plainnat}
\bibliography{refs}

\end{document}